\documentclass{article}

\usepackage{arxiv}

\usepackage[utf8]{inputenc}
\usepackage[T1]{fontenc}
\usepackage{hyperref}
\usepackage{url}
\usepackage{booktabs}
\usepackage{amsmath,amssymb,amsfonts,amsthm}
\usepackage{nicefrac}
\usepackage{microtype}
\usepackage{cleveref}
\usepackage{graphicx}
\usepackage{capt-of}
\usepackage{natbib}
\usepackage{doi}

\usepackage{bm}
\usepackage{xcolor}
\usepackage{colortbl}
\usepackage{multirow}
\usepackage{makecell}
\usepackage{array}
\usepackage{mathrsfs}

\newtheorem{theorem}{Theorem}[section]

\newtheorem{assumption}[theorem]{Assumption}
\newtheorem{proposition}[theorem]{Proposition}

\newcommand{\best}[1]{\colorbox{gray!30}{\textbf{#1}}}
\newcommand{\secb}[1]{\underline{#1}}

\title{LaST-SR: Laplace-Inspired Steady--Transient Complex-Frequency Decomposition for Single Image Super-Resolution}
\author{
Linhao Li \qquad Zhaojie Pan \qquad Langkun Chen\\
School of Mathematics\\
Southeast University\\
Nanjing, China
}

\renewcommand{\shorttitle}{LaST-SR}

\hypersetup{
hidelinks,
pdftitle={LaST-SR: Laplace-Inspired Steady--Transient Complex-Frequency Decomposition for Single Image Super-Resolution},
pdfauthor={Linhao Li, Zhaojie Pan, Langkun Chen},
pdfkeywords={single-image super-resolution, Laplace neural operator, complex-frequency decomposition},
}

\begin{document}

\maketitle

\begin{abstract}
Single-image super-resolution (SISR) requires global context modeling for structurally consistent reconstruction. Fourier operators are increasingly adopted for global feature modeling. However, their periodic spectral bases constrain the representation of localized aperiodic variations, limiting the recovery of irregular structures and fine details. In dynamical systems, the Laplace neural operator extends Fourier modes to complex frequencies and decomposes the output signal into complementary steady-state and transient responses to jointly model periodic and aperiodic information. We derive, for the first time, an approximate steady--transient decomposition for two-dimensional feature maps, providing an analytical basis for the proposed complex-frequency decomposition. Accordingly, we propose LaST-SR, centered on a Complex-Frequency Decomposition module that couples a global full-spectrum Fourier branch for image-wide dependencies and long-range structural consistency with a window-conditioned local complex-frequency branch for localized, content-dependent aperiodic variations. To fuse the resulting features, we further design a Steady--Transient Collaborative Aggregation module for cross-branch interaction and joint aggregation. Experiments on five benchmarks show that LaST-SR achieves the best PSNR/SSIM among the compared methods for $\times2$ and $\times4$ SISR. Ablation studies further validate the effectiveness of the proposed architecture and its key modeling mechanisms.
\end{abstract}

\section{Introduction}

Single-image super-resolution (SISR) aims to reconstruct a high-resolution (HR) image from a degraded low-resolution (LR) observation, offering a computational alternative when direct acquisition of HR images is impractical. Since the degradation process discards or aliases high-frequency information, multiple HR images may correspond to the same LR observation, rendering SISR an inherently ill-posed inverse problem.

To constrain this ill-posed reconstruction problem, SISR models must infer missing high-frequency information by jointly exploiting local detail cues and long-range structural dependencies. CNN-based methods effectively extract local patterns, but their predominantly local operations make explicit image-wide structural modeling difficult~\cite{SRCNN,VDSR,DRCN,EDSR,RDN,CARN}. Transformer-based methods strengthen contextual interaction through self-attention, yet the quadratic cost of global attention often restricts such interaction to local windows or regions~\cite{SwinIR,CAT,NGswin,HAT,OmniSR}. Against this background, Fourier-domain operators have been increasingly adopted for efficient image-wide feature interaction through spectral-domain modeling~\cite{SwinFIR,NL-FFC,FDSR,liu2025difffno}. However, their globally supported periodic modes and spatially shared spectral modulation remain insufficiently adaptive to localized, content-dependent aperiodic variations.

Laplace-domain representation offers a principled direction for addressing this limitation. The Laplace transform generalizes the Fourier transform by extending the purely imaginary frequency $\mathrm{i}\omega$ to the complex frequency $s=\sigma+i\omega$, thereby augmenting oscillatory Fourier modes with exponential modulation. Building on this extension, the Laplace neural operator (LNO) employs a pole--residue formulation to decompose the dynamical output signal into complementary steady-state and transient responses, enabling the modeling of non-periodic signals and transient behavior beyond purely oscillatory Fourier modes~\cite{LNO}. This motivates the use of complex-frequency modeling for SISR, where exponentially modulated modes provide additional flexibility for representing localized aperiodic variations. However, this decomposition is derived for time-dependent dynamical systems, and a corresponding formulation for static two-dimensional feature maps has not yet been established.

To address this gap, we derive an approximate steady--transient decomposition for static two-dimensional feature maps based on a Laplace-domain formulation. The two responses are then interpreted as having complementary spatial roles: the steady response models image-wide dependencies and long-range structural consistency, whereas the transient response captures localized, content-dependent aperiodic variations. Here, steady and transient denote complementary spatial responses rather than temporal dynamics. Guided by this formulation, we propose the Steady--Transient Super-Resolution Network (LaST-SR). Its core Complex-Frequency Decomposition (CFD) module implements the two responses using a global full-spectrum Fourier branch and a window-conditioned local complex-frequency branch, respectively, while the Steady--Transient Collaborative Aggregation (STCA) module enables cross-branch interaction and joint aggregation. The main contributions of this work are summarized as follows:

\begin{itemize}
    \item To the best of our knowledge, we are the first to derive an approximate steady--transient decomposition for static two-dimensional feature maps, under certain assumptions. This formulation provides a spatial counterpart to the Laplace-domain pole--residue interpretation of time-dependent dynamical systems and a theoretical grounding for Laplace-inspired feature modeling in SISR.

    \item We propose LaST-SR, a framework centered on the Complex-Frequency Decomposition (CFD) module. CFD models the steady and transient responses via two complementary branches: a global full-spectrum Fourier branch for image-wide dependencies and long-range structural consistency, and a window-conditioned local complex-frequency branch for localized, content-dependent aperiodic variations. The Steady--Transient Collaborative Aggregation (STCA) module further enables cross-branch interaction and joint aggregation of the resulting steady and transient features.

    \item Extensive experiments on five standard benchmarks demonstrate that LaST-SR achieves the best PSNR/SSIM among the compared methods for both $\times2$ and $\times4$ SISR. Ablation studies further validate the proposed decomposition, the complementary roles of the two branches, and the effectiveness of STCA.
\end{itemize}

\section{Related Work}
\label{sec:related_work}

\subsection{Frequency-Domain SISR Methods}
\label{subsec:frequency_restoration}

Frequency-domain modeling has been increasingly explored in SISR to capture image-wide dependencies efficiently. SwinFIR~\cite{SwinFIR} and NL-FFC~\cite{NL-FFC} incorporate Fourier convolution to enhance global feature interaction. FSR~\cite{li2023fsr} combines transform-domain global features with local spatial features and processes different wavelet subbands using capacity-adaptive branches. FDSR~\cite{FDSR} decomposes frequency components for progressive reconstruction, whereas DiffFNO~\cite{liu2025difffno} introduces a weighted Fourier neural operator for arbitrary-scale SISR. Despite their different implementations, these methods largely rely on globally shared spectral modulation or predefined frequency decomposition, leaving localized and content-dependent aperiodic variations insufficiently modeled.

\subsection{Complex-Frequency Modeling}
\label{subsec:laplace_inspired}

Complex-exponential models extend purely oscillatory Fourier modes with exponential modulation. Classical methods, including Prony-type methods, Matrix Pencil, and ESPRIT, estimate poles and modal coefficients for signal decomposition and reconstruction~\cite{prony_ss,Matrixpencil,ESPRIT}. In neural operator learning, LNO~\cite{LNO} introduces a learnable pole--residue parameterization to represent Laplace-domain kernels and separates dynamical responses into steady-state and transient components. However, the steady-state--transient formulation in LNO is developed for temporal signals and time-dependent systems, while its adaptation to static two-dimensional feature modeling in SISR remains underexplored.

\section{Theoretical Analysis}
\label{sec:setup}

Let $\mathbf{X}\in\mathbb{R}^{B\times H\times W\times C_{\mathrm{mid}}}$ denote an intermediate feature tensor. For the $b$-th sample, we represent the discrete feature vectors $\mathbf{X}_{b,h,w,:}$ as samples of a continuous vector-valued function $\mathbf{v}_b:\Omega\to\mathbb{R}^{C_{\mathrm{mid}}}$, $\mathbf{v}_b(x_w,y_h)=\mathbf{X}_{b,h,w,:}$, where the regular grid points $(x_w,y_h)\in\Omega$, and $\Omega=[0,T_x)\times[0,T_y)$. For notational simplicity, we fix one sample and one feature channel, denote the corresponding scalar component by $v:\Omega\to\mathbb{R}$.

To characterize the global spectral structure of $v$, we adopt a truncated Fourier representation on the finite spatial domain under periodic extension:
\begin{equation}
\label{eq:fourier_main}
v(x,y)\approx\sum_{p=-P}^{P}\sum_{q=-Q}^{Q}\alpha_{pq}\exp(\mathrm{i}\omega_p^x x)\exp(\mathrm{i}\omega_q^y y),
\end{equation}
where $\alpha_{pq}\in\mathbb{C}$, $\omega_p^x=2\pi p/T_x$, and $\omega_q^y=2\pi q/T_y$. Since $v$ is real-valued, the Fourier coefficients satisfy $\alpha_{-p,-q}=\overline{\alpha_{pq}}$.

Let $u$ denote the response of a translation-invariant linear convolution operator with a learnable kernel $\kappa$ acting on $v$:
\begin{equation}
\label{eq:convolution_main}
u(x,y)=(\kappa*v)(x,y),\qquad (x,y)\in\Omega.
\end{equation}
Following the complex-exponential modeling paradigm of Prony-SS~\cite{prony_ss}, we parameterize the learnable kernel as
\begin{equation}
\label{eq:kernel_main}
\kappa(x,y)=\sum_{m=1}^{K_x}\sum_{n=1}^{K_y}\mathcal{R}_{mn}\exp\!\left(\mu_{m}^{x}x+\mu_{n}^{y}y\right),
\end{equation}
where $\mu_{m}^{x},\mu_{n}^{y}\in\mathbb{C}$ are learnable complex poles along the two spatial directions and $\mathcal{R}_{mn}\in\mathbb{C}$ is the residue coefficient associated with each pole pair. Under a positive-quadrant extension of $v$ and a first-quadrant-supported kernel $\kappa$, the two-dimensional one-sided Laplace convolution theorem gives $U(s_x,s_y)=K(s_x,s_y)V(s_x,s_y)$. This product yields four pole families after partial-fraction expansion. Retaining the Fourier--Fourier and complex-pole--complex-pole families while collecting the two mixed families into an approximation residual gives
\begin{equation}
\label{eq:decomposition_main}
\begin{split}
u(x,y)\approx{}&\sum_{p=-P}^{P}\sum_{q=-Q}^{Q}\lambda_{pq}\exp(\mathrm{i}\omega_p^x x+\mathrm{i}\omega_q^y y)\\
&+\sum_{m=1}^{K_x}\sum_{n=1}^{K_y}\gamma_{mn}\exp(\mu_{m}^{x}x+\mu_{n}^{y}y).
\end{split}
\end{equation}
The first term denotes the steady response $u_{\mathrm{ss}}(x,y)$, and the second denotes the transient response $u_{\mathrm{tr}}(x,y)$, where $\lambda_{pq}$ and $\gamma_{mn}$ are their corresponding effective coefficients.

\paragraph{Remark.}
The exact two-dimensional expansion also contains Fourier--complex-pole and complex-pole--Fourier families, which are not assumed to vanish. Equation~\eqref{eq:decomposition_main} retains the two same-type families and collects the mixed families into a residual, yielding a task-oriented dual-response approximation. See Sec.~A.2 of the technical supplement for the complete derivation.

In SISR, we assign complementary functional roles to the two retained responses. The steady response uses globally supported oscillatory modes to capture image-wide dependencies and long-range structural consistency. The transient response employs general complex-frequency modes to represent content-dependent aperiodic variations, while the windowed realization introduced below further provides local adaptation.

Under the periodic-extension convention, the periodic counterpart of Eq.~\eqref{eq:convolution_main} admits a mode-wise Fourier spectral-multiplier form with the same basis family as the steady response in Eq.~\eqref{eq:decomposition_main}. This correspondence motivates global Fourier-domain steady modeling; see Sec.~A.3 of the technical supplement.

For noncoincident Fourier and learnable poles, residue evaluation yields the coefficient of the retained transient response as
\begin{equation}
\label{eq:gamma_main}
\gamma_{mn}=\mathcal{R}_{mn}\sum_{p=-P}^{P}\sum_{q=-Q}^{Q}\frac{\alpha_{pq}}{(\mu_{m}^{x}-\mathrm{i}\omega_p^x)(\mu_{n}^{y}-\mathrm{i}\omega_q^y)}.
\end{equation}
As shown in Table~\ref{tab:transient_complexity}, direct full-map evaluation materializes large complex-valued intermediate tensors and incurs substantial peak memory. Moreover, modal coefficients derived from the full-map spectrum are shared across all spatial locations, limiting their adaptation to region-specific variations. We therefore implement the transient response within normalized local windows, sharing poles and residues across windows while deriving window-specific modal coefficients from each local spectrum.

\section{LaST-SR Architecture}
\label{sec:architecture}
\begin{figure*}[t]
    \centering
    \includegraphics[width=\textwidth]{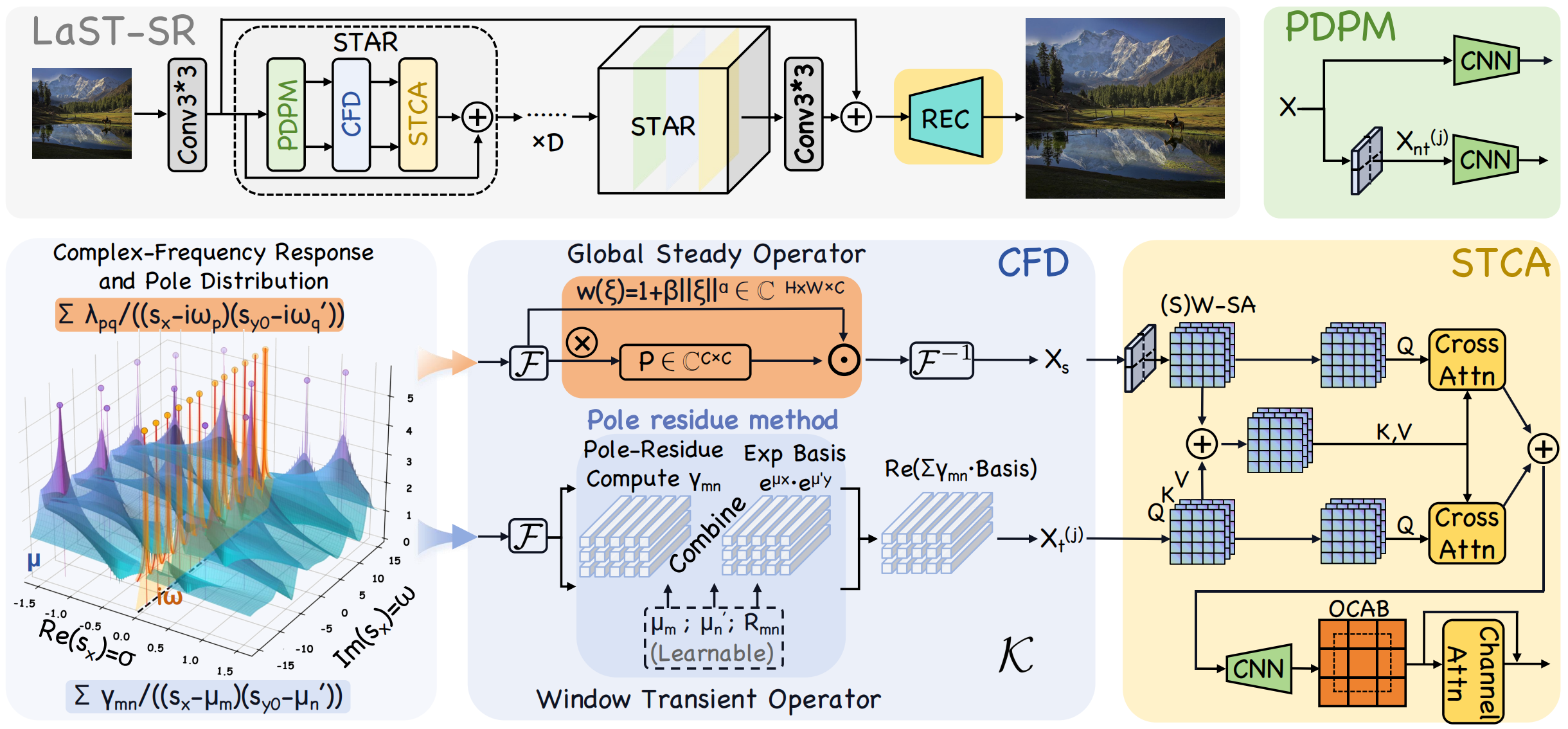}
    \caption{Overall architecture of LaST-SR. A shallow feature extractor is followed by \(D\) STAR blocks and a reconstruction module. Each STAR block comprises PDPM for branch-specific projection, CFD for global steady and windowed local transient modeling, and STCA for bidirectional cross-branch interaction and aggregation. The lower-left panel illustrates the Laplace-domain pole--residue interpretation: poles on the imaginary axis induce purely oscillatory steady responses, whereas poles with nonzero real parts yield transient responses with exponential envelopes.}
    \label{fig:visual_structure}
\end{figure*}
Guided by the approximate steady--transient decomposition, we develop LaST-SR, whose overall architecture is illustrated in Fig.~\ref{fig:visual_structure}. Its core CFD module models the two response families using a global full-spectrum Fourier branch and a window-conditioned local complex-frequency branch, respectively. PDPM provides branch-specific low-dimensional representations before CFD, while STCA enables cross-branch interaction and joint aggregation.

LaST-SR comprises shallow feature extraction, deep feature modeling with \(D\) sequential Steady--Transient Attention Representation (STAR) blocks, and high-resolution reconstruction:
\begin{equation}
\label{eq:overall_pipeline}
\begin{aligned}
X_0 &= H_{\mathrm{SF}}(I_{\mathrm{LR}}),\\
X_d &= H_{\mathrm{STAR}}^{(d)}(X_{d-1}) + X_{d-1}, \qquad d=1,\ldots,D,\\
I_{\mathrm{SR}} &= H_{\mathrm{Rec}}(X_D + X_0),
\end{aligned}
\end{equation}
where \(H_{\mathrm{SF}}\) is a \(3\times3\) convolution, \(H_{\mathrm{STAR}}^{(d)}\) denotes the \(d\)-th STAR block, and \(H_{\mathrm{Rec}}\) consists of a \(3\times3\) convolution followed by PixelShuffle. Each STAR block contains PDPM, CFD, and STCA: PDPM produces branch-specific low-dimensional features, CFD models the global steady and windowed local transient responses, and STCA enables their cross-branch interaction and joint aggregation.
\begin{table}[t]
\centering
\footnotesize
\setlength{\tabcolsep}{7pt}
\begin{tabular}{lcc}
\toprule
\textbf{Method}
& \textbf{Peak Mem. (MB)}
& \textbf{Latency (ms)}
\\
\midrule
Global
& $\geq 4{,}529.85^{\dagger}$
& OOM
\\
Windowed (Ours)
& 66.79
& 0.856
\\
\bottomrule
\end{tabular}
\caption{Operator-level memory and latency comparison between the global and windowed transient implementations under the $\times4$ SR setting, using an LR feature resolution corresponding to a 720p ($1280\times720$) output. Both implementations have identical parameter counts and leading-order arithmetic complexity. $\dagger$ The reported global memory is a conservative analytical lower bound determined by its largest complex-valued intermediate tensor, as the full-map forward pass did not complete.}
\label{tab:transient_complexity}
\end{table}

\subsection{PDPM: Pre-Decomposition Projection Module}
\label{sec:pdpm}

PDPM generates low-dimensional feature representations for the steady and transient branches. Given \(X_{\mathrm{in}}\in\mathbb{R}^{B\times HW\times C}\), an independent two-layer pointwise projection is applied to each branch \(\tau\in\{\mathrm{ss},\mathrm{tr}\}\):
\begin{equation}
\label{eq:pdpm_projection}
    Z_{\tau}=\phi\!\left(\mathcal{P}_{\tau}^{(2)}\left(\phi\!\left(\mathcal{P}_{\tau}^{(1)}\left(\mathrm{LN}(X_{\mathrm{in}})\right)\right)\right)\right),
\end{equation}
where \(\mathrm{LN}\) denotes LayerNorm, \(\phi\) denotes LeakyReLU, and \(\mathcal{P}_{\tau}^{(i)}\) denotes a pointwise channel projection implemented by a \(1\times1\) convolution. The two branches do not share parameters and produce \(Z_{\tau}\in\mathbb{R}^{B\times HW\times C_{\mathrm{mid}}}\), where \(C_{\mathrm{mid}}<C\).

The steady representation \(Z_{\mathrm{ss}}\) is reshaped into \(X_{\mathrm{ss}}\in\mathbb{R}^{B\times C_{\mathrm{mid}}\times H\times W}\) to enable global steady modeling. In parallel, \(Z_{\mathrm{tr}}\) is partitioned into non-overlapping \(M\times M\) windows, yielding \(X_{\mathrm{tr}}\in\mathbb{R}^{BN_w\times C_{\mathrm{mid}}\times M\times M}\) for local transient modeling, where \(N_w=HW/M^2\).

\subsection{CFD: Complex-Frequency Decomposition}
\label{sec:cfd}

CFD instantiates the steady and transient responses in Eq.~\eqref{eq:decomposition_main} using a global full-spectrum Fourier branch and a windowed local complex-frequency branch, respectively. 

The lower-left panel of Fig.~\ref{fig:visual_structure} visualizes a response slice by fixing \(s_y=s_{y_0}=0.15+3.80\mathrm{i}\) and varying \(s_x=\sigma+\mathrm{i}\omega\). For a complex pole \(\mu\), the inverse-transform basis is \(e^{\mu x}=e^{\operatorname{Re}(\mu)x}e^{\mathrm{i}\operatorname{Im}(\mu)x}\). Fourier poles lie on the imaginary axis and yield purely oscillatory modes, whereas learnable poles may have nonzero real parts that introduce spatial growth or decay envelopes. The surface and scatter plots show the response magnitude and pole distribution, respectively.

\subsubsection{Global Steady Branch}
\label{sec:cfd_steady}

Given $X_{\mathrm{ss}}$, the steady branch first computes $\widehat{X}_{\mathrm{ss}}(\xi)=\mathcal{F}(X_{\mathrm{ss}})(\xi)$, where $\xi$ denotes a two-dimensional discrete frequency. To preserve global spectral organization and fine-scale structural components, we retain and modulate the complete discrete spectrum. The modulated spectrum is given by
\begin{equation}
\label{eq:steady_modulation}
\widehat{X}_s(\xi)=P_{\mathrm{ss}}\!\left[w(\xi)\widehat{X}_{\mathrm{ss}}(\xi)\right],
\end{equation}
where $P_{\mathrm{ss}}\in\mathbb{C}^{C_{\mathrm{mid}}\times C_{\mathrm{mid}}}$ is a complex-valued channel-mixing matrix shared across frequencies, and $w(\xi)=1+\eta\|\xi\|^{\epsilon}$ is a frequency-dependent weight with learnable scalar $\eta$ and fixed exponent $\epsilon$. The steady feature is obtained by
\begin{equation}
\label{eq:steady_reconstruction}
X_s=\operatorname{Re}\!\left[\mathcal{F}^{-1}(\widehat{X}_s)\right].
\end{equation}

\subsubsection{Local Transient Branch}
\label{sec:cfd_transient}

The local transient branch implements the transient response in Eq.~\eqref{eq:decomposition_main} by discretizing the coefficient relation in Eq.~\eqref{eq:gamma_main} within normalized local windows. Complex poles and residue parameters are shared across windows, while modal coefficients are dynamically generated from each local spectrum. For the $r$-th window, the local Fourier coefficients are computed as
\begin{equation}
\label{eq:local_spectrum}
    \alpha^{(r)}(\omega^x,\omega^y)=\mathcal{F}(X_{\mathrm{tr}}^{(r)})(\omega^x,\omega^y),
\end{equation}
where \(\alpha^{(r)}\in\mathbb{C}^{C_{\mathrm{mid}}\times M\times M}\), and \((\omega^x,\omega^y)\) denotes the discrete frequency grid within the window.

Learnable complex poles and residues are parameterized for each input--output channel pair:
\begin{equation}
\label{eq:transient_parameters}
\begin{aligned}
    \mu^x&\in\mathbb{C}^{C_{\mathrm{mid}}\times C_{\mathrm{mid}}\times K_x},\qquad
    \mu^y\in\mathbb{C}^{C_{\mathrm{mid}}\times C_{\mathrm{mid}}\times K_y},\\
    \mathcal{R}&\in\mathbb{C}^{C_{\mathrm{mid}}\times C_{\mathrm{mid}}\times K_x\times K_y},
\end{aligned}
\end{equation}
where \(K_x\) and \(K_y\) denote the numbers of complex-frequency modes along the two spatial directions. For channel pair \((c,d)\) and mode pair \((m,n)\), the pole--residue spectral response is
\begin{equation}
\label{eq:phi_def}
    \Phi_{c,d,m,n}(\omega^x,\omega^y)=\frac{\mathcal{R}_{c,d,m,n}}{(\mu^x_{c,d,m}-\mathrm{i}\omega^x)(\mu^y_{c,d,n}-\mathrm{i}\omega^y)}.
\end{equation}
The modal coefficients for the \(r\)-th window are computed as
\begin{equation}
\label{eq:gamma_compute}
    \gamma^{(r)}_{c,d,m,n}=\sum_{\omega^x,\omega^y}\alpha^{(r)}_c(\omega^x,\omega^y)\Phi_{c,d,m,n}(\omega^x,\omega^y).
\end{equation}
The transient response in the \(r\)-th window is reconstructed over normalized local coordinates \((\tilde{x},\tilde{y})\in[0,1]^2\):
\begin{equation}
\label{eq:transient_recon}
    X^{(r)}_{t,d}(\tilde{x},\tilde{y})=\operatorname{Re}\!\left[\sum_{c,m,n}\gamma^{(r)}_{c,d,m,n}\exp\!\left(\mu^x_{c,d,m}\tilde{x}+\mu^y_{c,d,n}\tilde{y}\right)\right].
\end{equation}
The window responses are reassembled into \(X_t\in\mathbb{R}^{B\times C_{\mathrm{mid}}\times H\times W}\). The imaginary and real pole components govern local oscillation and spatial growth or decay, respectively. Since image coordinates are spatial rather than temporal, we do not impose a negative-real-part constraint. Instead, normalized local coordinates mitigate coordinate-range-induced exponential amplification.

\begin{figure*}[!t]
    \centering

    % ================= Quantitative comparison table =================
    {
    \small
    \setlength{\tabcolsep}{1.15pt}

    \begin{tabular}{@{}lccccccc@{}}
        \toprule[1.0pt]
        \multirow{2}{*}{Method}
        & \multirow{2}{*}{Scale}
        & \multirow{2}{*}{Params.}
        & Set5
        & Set14
        & BSD100
        & Urban100
        & Manga109 \\
        \cmidrule(lr){4-8}
        & & (K)
        & PSNR/SSIM
        & PSNR/SSIM
        & PSNR/SSIM
        & PSNR/SSIM
        & PSNR/SSIM \\
        \midrule

        % ================= x2 Section =================
        CARN
        & \multirow{13}{*}{$\times2$}
        & 1,592
        & 37.76/0.9590
        & 33.52/0.9166
        & 32.09/0.8978
        & 31.92/0.9256
        & 38.36/0.9754 \\

        LAPAR-A
        & & 548
        & 38.01/0.9605
        & 33.62/0.9183
        & 32.19/0.8999
        & 32.10/0.9283
        & 38.67/0.9772 \\

        DeFiANs
        & & 1,028
        & 38.03/0.9605
        & 33.63/0.9181
        & 32.32/0.8999
        & 32.20/0.9286
        & 38.91/0.9775 \\

        SwinIR-light
        & & 910
        & 38.14/0.9611
        & 33.86/0.9206
        & 32.31/0.9012
        & 32.76/0.9340
        & 39.12/0.9783 \\

        SwinIR-NG
        & & 1,181
        & 38.17/0.9612
        & 33.94/0.9205
        & 32.31/0.9013
        & 32.78/0.9340
        & 39.20/0.9781 \\

        SwinFIR-T
        & & 872
        & 38.22/0.9614
        & 34.03/\secb{0.9215}
        & 32.37/0.9020
        & 33.03/0.9362
        & 39.30/0.9785 \\

        Omni-SR
        & & 772
        & 38.22/0.9613
        & 33.98/0.9210
        & 32.36/0.9020
        & 33.05/0.9363
        & 39.28/0.9784 \\

        FDSR
        & & 6,290
        & 38.20/0.9613
        & 33.90/0.9200
        & 32.33/0.9011
        & 32.72/0.9336
        & 39.09/0.9781 \\

        SRConvNet-L
        & & 885
        & 38.14/0.9610
        & 33.81/0.9199
        & 32.28/0.9010
        & 32.59/0.9321
        & 39.22/0.9779 \\

        CRAFT
        & & 737
        & 38.23/\secb{0.9615}
        & 33.92/0.9211
        & 32.33/0.9016
        & 32.86/0.9343
        & \secb{39.39}/\best{0.9786} \\

        LAMNet-large
        & & 1,024
        & \secb{38.27}/\secb{0.9615}
        & \secb{34.07}/0.9214
        & \secb{32.38}/\secb{0.9023}
        & \secb{33.16}/\secb{0.9373}
        & 39.38/0.9785 \\

        MaIR-Small
        & & 1,355
        & 38.20/0.9611
        & 33.91/0.9209
        & 32.34/0.9016
        & 32.97/0.9359
        & 39.32/0.9779 \\

        Ours
        & & 1,329
        & \best{38.32}/\best{0.9617}
        & \best{34.13}/\best{0.9225}
        & \best{32.41}/\best{0.9026}
        & \best{33.25}/\best{0.9380}
        & \best{39.43}/\best{0.9786} \\

        \midrule

        % ================= x4 Section =================
        CARN
        & \multirow{13}{*}{$\times4$}
        & 1,592
        & 32.13/0.8937
        & 28.60/0.7806
        & 27.58/0.7349
        & 26.07/0.7837
        & 30.45/0.9073 \\

        LAPAR-A
        & & 659
        & 32.15/0.8944
        & 28.61/0.7818
        & 27.61/0.7366
        & 26.14/0.7871
        & 30.42/0.9074 \\

        DeFiANs
        & & 1,065
        & 32.16/0.8942
        & 28.63/0.7810
        & 27.58/0.7363
        & 26.10/0.7862
        & 30.59/0.9084 \\

        SwinIR-light
        & & 930
        & 32.44/0.8976
        & 28.77/0.7858
        & 27.69/0.7406
        & 26.47/0.7980
        & 30.92/0.9151 \\

        SwinIR-NG
        & & 1,201
        & 32.44/0.8980
        & 28.83/0.7870
        & 27.73/0.7418
        & 26.61/0.8010
        & 31.09/0.9161 \\

        SwinFIR-T
        & & 967
        & 32.49/0.8991
        & 28.86/0.7879
        & 27.74/0.7425
        & 26.72/0.8076
        & 31.23/0.9178 \\

        Omni-SR
        & & 792
        & 32.49/0.8988
        & 28.78/0.7859
        & 27.71/0.7415
        & 26.64/0.8018
        & 31.02/0.9151 \\

        FDSR
        & & 8,720
        & 32.56/\secb{0.8995}
        & 28.85/0.7881
        & 27.75/0.7415
        & 26.68/0.8030
        & 31.24/0.9174 \\

        SRConvNet-L
        & & 902
        & 32.44/0.8976
        & 28.77/0.7857
        & 27.69/0.7402
        & 26.47/0.7970
        & 30.96/0.9139 \\

        CRAFT
        & & 753
        & 32.52/0.8989
        & 28.85/0.7872
        & 27.72/0.7418
        & 26.56/0.7995
        & 31.18/0.9168 \\

        LAMNet-large
        & & 1,045
        & 32.60/\best{0.9000}
        & 28.87/\secb{0.7886}
        & \secb{27.77}/\secb{0.7434}
        & \secb{26.82}/\secb{0.8078}
        & 31.33/\secb{0.9183} \\

        MaIR-Small
        & & 1,374
        & \secb{32.62}/0.8992
        & \secb{28.90}/0.7882
        & \secb{27.77}/0.7431
        & 26.73/0.8049
        & \secb{31.34}/\secb{0.9183} \\

        Ours
        & & 1,516
        & \best{32.65}/\best{0.9000}
        & \best{28.93}/\best{0.7895}
        & \best{27.80}/\best{0.7440}
        & \best{26.87}/\best{0.8081}
        & \best{31.43}/\best{0.9189} \\

        \bottomrule[1.0pt]
    \end{tabular}
    }

    % ================= Table caption =================
    \begingroup
        \captionof{table}{PSNR/SSIM comparison with existing SISR methods on five benchmarks. Best and second-best results are highlighted in bold and underlined, respectively.}
        \label{tab:sr_comparison}
    \endgroup

    % ================= Qualitative comparison figure =================
    \includegraphics[width=\textwidth]{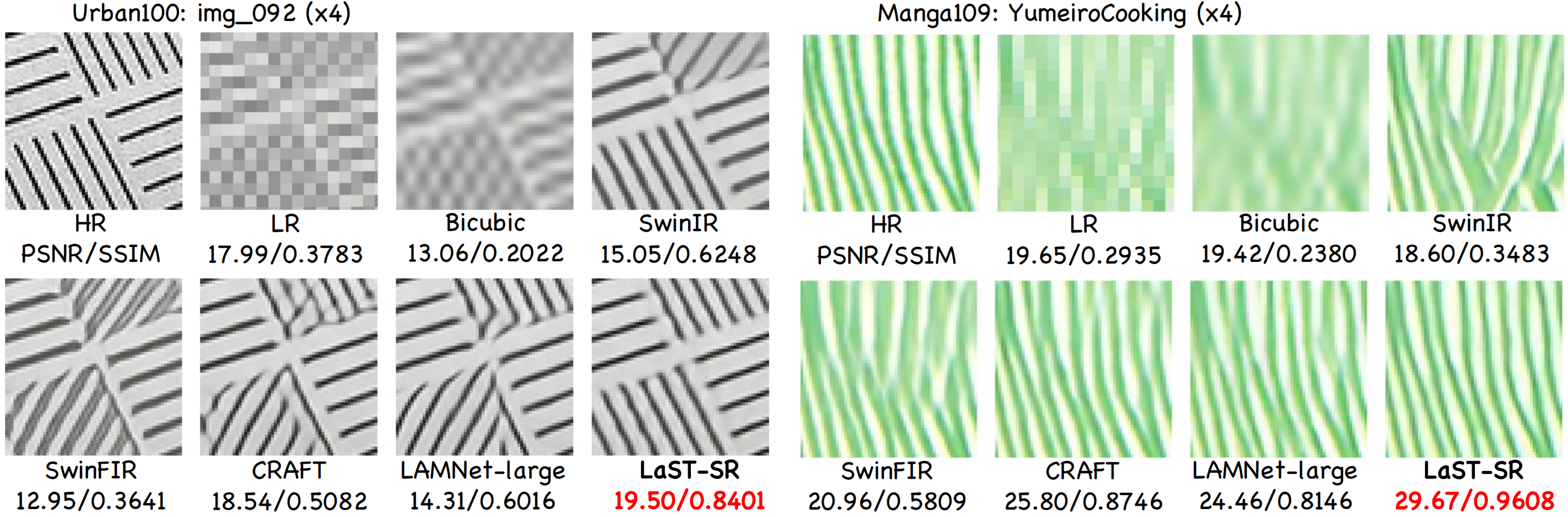}

    % ================= Figure caption =================
    \begingroup
        \caption{Qualitative comparison for \(\times4\) SISR on image 092 from Urban100 and \textit{YumeiroCooking} from Manga109. LaST-SR better preserves line continuity, local junctions, and irregular stripe patterns. Best PSNR/SSIM scores are highlighted in red.}
        \label{fig:visual_comparison}
    \endgroup

\end{figure*}

\subsection{STCA: Steady--Transient Collaborative Aggregation}
\label{sec:stca}

STCA constructs shared guidance from the steady and transient features and symmetrically refines both branches. Given the CFD outputs \(X_s\) and \(X_t\), each branch is first enhanced by window-based self-attention:
\begin{equation}
\label{eq:stca_branch_enhancement}
\begin{aligned}
    S
    &=
    X_s+
    \operatorname{WMSA}_{\delta}
    \!\left(\operatorname{LN}(X_s)\right),\\
    T
    &=
    X_t+
    \operatorname{WMSA}_{\delta}
    \!\left(\operatorname{LN}(X_t)\right).
\end{aligned}
\end{equation}
Here, \(\delta\) alternates between \(0\) and \(\lfloor M/2\rfloor\) across successive STAR blocks, corresponding to regular and shifted windows, respectively. The enhanced branches are fused to form a shared guidance feature:
\begin{equation}
\label{eq:stca_guidance}
    G
    =
    \operatorname{WMSA}_{\delta}
    \!\left(
    \mathcal{C}_{3\times3}(S+T)
    \right).
\end{equation}

Using \(S\) and \(T\) as branch-specific queries and \(G\) as the shared key and value, window-based cross-attention is performed as
\begin{equation}
\label{eq:stca_cross_attention}
\begin{aligned}
    \widehat{S}
    &=
    \operatorname{WCA}_{\delta}(S,G),\\
    \widehat{T}
    &=
    \operatorname{WCA}_{\delta}(T,G),
\end{aligned}
\end{equation}
where \(\operatorname{WCA}_{\delta}(\mathbf{U},\mathbf{Z})\) uses \(\mathbf{U}\) as the query and \(\mathbf{Z}\) as both the key and value. After feed-forward refinement, the two branches are restored to the spatial layout and aggregated:
\begin{equation}
\label{eq:stca_branch_fusion}
\begin{aligned}
    S'
    &=
    \widehat{S}
    +
    \operatorname{FFN}
    \!\left(
    \operatorname{LN}(\widehat{S})
    \right),\\
    T'
    &=
    \widehat{T}
    +
    \operatorname{FFN}
    \!\left(
    \operatorname{LN}(\widehat{T})
    \right),\\
    F
    &=
    \operatorname{WinRev}_{\delta}(S')
    +
    \operatorname{WinRev}_{\delta}(T').
\end{aligned}
\end{equation}

The fused feature is projected back to the original channel dimension:
\begin{equation}
\label{eq:stca_projection}
    X_{\mathrm{fuse}}
    =
    \phi\!\left(
    \mathcal{P}_{f}^{(2)}
    \!\left(
    \phi\!\left(
    \mathcal{P}_{f}^{(1)}(F)
    \right)
    \right)
    \right),
\end{equation}
where \(\mathcal{P}_{f}^{(1)}\) and \(\mathcal{P}_{f}^{(2)}\) are \(1\times1\) convolutions that jointly restore the channel dimension from \(C_{\mathrm{mid}}\) to \(C\). Finally, an overlapping cross-attention block (OCAB)~\cite{HAT} alleviates window-boundary effects, followed by channel-attention recalibration:
\begin{equation}
\label{eq:stca_output}
\begin{aligned}
    Y
    &=
    \operatorname{OCAB}(X_{\mathrm{fuse}}),\\
    X_{\mathrm{out}}
    &=
    Y
    +
    \operatorname{ChAttn}(Y).
\end{aligned}
\end{equation}
Through shared guidance, STCA jointly refines image-wide structural information and localized aperiodic variations.

\subsection{Training Objective}
\label{sec:objective}

We combine spatial reconstruction and Fourier-magnitude consistency losses:
\begin{equation}
\label{eq:loss_total}
\begin{aligned}
    \mathcal{L}_{\mathrm{total}}
    &=\mathcal{L}_{\mathrm{spatial}}+\lambda_f\mathcal{L}_{\mathrm{freq}},\\
    \mathcal{L}_{\mathrm{spatial}}
    &=\frac{1}{N}\left\|I_{\mathrm{SR}}-I_{\mathrm{HR}}\right\|_1,\\
    \mathcal{L}_{\mathrm{freq}}
    &=\frac{1}{N_f}\left\|
    \left|\mathcal{F}_{2D}(I_{\mathrm{SR}})\right|
    -\left|\mathcal{F}_{2D}(I_{\mathrm{HR}})\right|
    \right\|_1,
\end{aligned}
\end{equation}
where \(N\) and \(N_f\) are the numbers of spatial and frequency elements, and \(\lambda_f\) weights the frequency loss. The two terms provide complementary pixel- and spectral-domain supervision for the final reconstruction.
\begin{figure*}[t]
    \centering
    \includegraphics[width=0.95\linewidth]{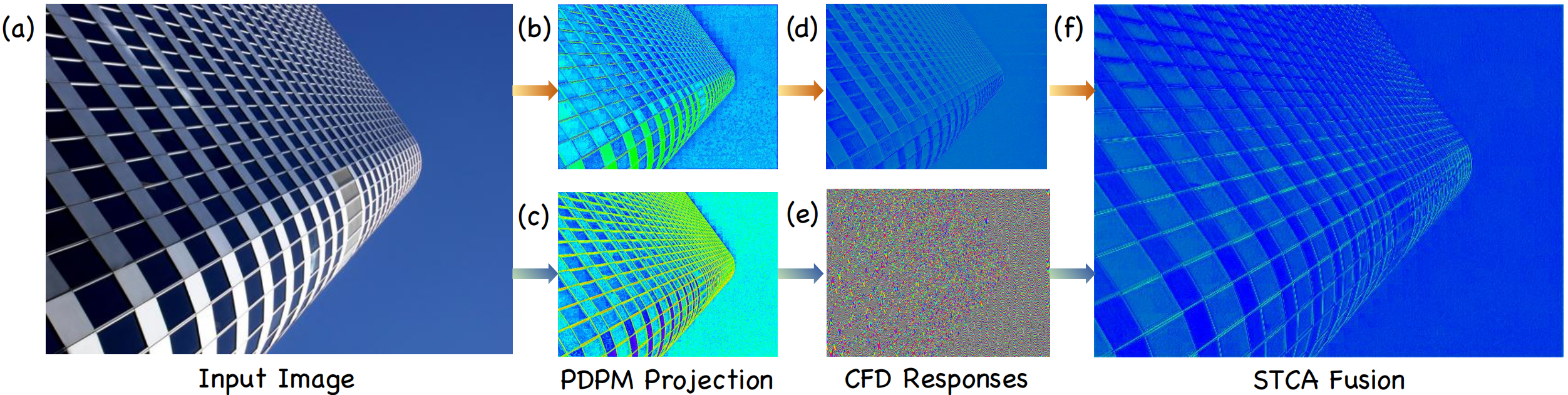}
    \caption{Visualization of intermediate representations in LaST-SR on Urban100 img\_005: (a) input image; (b,c) PDPM projections for the steady and transient branches, respectively; (d,e) the corresponding CFD responses; and (f) the STCA-fused representation.}

    \label{fig:feature_vis_005}
\end{figure*}

\begin{table}[t]
    \caption{Ablation study of the main components for \(\times2\) SISR on Urban100 and Manga109. In the corresponding variants, either the steady or transient operator is replaced with a convolutional operator, while cross-branch attention is replaced with independent self-attention. The best results are highlighted in bold.}
    \label{tab:ablation_x2}
    \centering
    \small
    \setlength{\tabcolsep}{6pt}

    \begin{tabular}{lcc}
        \toprule
        Variant
        & Urban100
        & Manga109 \\
        & PSNR/SSIM
        & PSNR/SSIM \\
        \midrule

        Steady $\rightarrow$ Conv
        & 28.12/0.8724
        & 33.25/0.9533 \\

        Transient $\rightarrow$ Conv
        & 33.14/0.9373
        & 39.35/0.9784 \\

        Cross $\rightarrow$ Self-Attn
        & 33.19/0.9376
        & 39.41/0.9784 \\

        Full model
        & \best{33.25}/\best{0.9380}
        & \best{39.43}/\best{0.9786} \\

        \bottomrule
    \end{tabular}
\end{table}

\section{Experiments}
\label{sec:experiments}

We evaluate LaST-SR under the standard bicubic degradation setting, comparing it with state-of-the-art methods and analyzing its architectural components and key modeling mechanisms.

\subsection{Experimental Settings}
\label{sec:experimental_settings}

\textbf{Datasets and metrics.} We train LaST-SR on the 800 training images of DIV2K~\cite{div2k} with random flipping and rotation. Evaluation is performed on Set5~\cite{set5}, Set14~\cite{set14}, BSD100~\cite{bsd100}, Urban100~\cite{urban100}, and Manga109~\cite{manga109}. LR images are generated by bicubic downsampling, and PSNR/SSIM are computed on the Y channel of YCbCr.

\noindent\textbf{Implementation details.} LaST-SR is implemented in PyTorch and trained on eight NVIDIA Tesla V100 GPUs with \(B=64\), \(C=72\), and \(D=6\) STAR blocks. PDPM projects each branch to \(C_{\mathrm{mid}}=8\) channels. The steady branch uses the full discrete spectrum with \(\epsilon=0.7\) and learnable \(\eta\), while the transient branch uses \(K_x=K_y=12\) complex-frequency modes. All window-based operations use \(M=16\).

We train on randomly cropped \(64\times64\) LR patches using Adam with \(\beta_1=0.9\) and \(\beta_2=0.99\). The initial learning rate is \(2\times10^{-4}\), and training lasts \(5\times10^5\) iterations. It is halved at \(3\times10^5\), \(4\times10^5\), \(4.5\times10^5\), and \(4.75\times10^5\) iterations. We set the loss weight \(\lambda_f=0.01\). FLOPs are measured on an LR input corresponding to an HR output of \(1280\times720\).

\subsection{Comparison with State-of-the-Art Methods}
\label{sec:comparison_sota}

We compare LaST-SR with representative CNN-, Transformer-, and frequency-domain SISR methods spanning comparable and larger parameter budgets, including CARN~\cite{CARN}, LAPAR-A~\cite{LAPAR-A}, DeFiANs~\cite{DeFiANs}, SwinIR-light~\cite{SwinIR}, SwinIR-NG~\cite{NGswin}, SwinFIR-T~\cite{SwinFIR}, Omni-SR~\cite{OmniSR}, FDSR~\cite{FDSR}, SRConvNet-L~\cite{SRConvNet}, CRAFT~\cite{CRAFT}, LAMNet-large~\cite{LAMNet}, and MaIR-Small~\cite{MaIR}.

As shown in Table~\ref{tab:sr_comparison}, LaST-SR achieves the best PSNR/SSIM on five benchmarks for both \(\times2\) and \(\times4\) SISR. On Urban100 and Manga109, LaST-SR improves upon the best competing results by \(0.09\) dB and \(0.04\) dB, respectively, for \(\times2\), and by \(0.05\) dB and \(0.09\) dB, respectively, for \(\times4\). Fig.~\ref{fig:visual_comparison} shows that existing methods often blur or disrupt local structures near bends, intersections, and directional changes. LaST-SR better preserves long-range structural continuity, coherent local arrangements, stripe orientations, and irregular textures across Urban100 and Manga109.

\subsection{Ablation Studies}
\label{sec:ablation}

We conduct ablations under the \(\times2\) setting to evaluate the steady and transient operators, cross-branch interaction, Fourier-magnitude loss, and pole parameterization. All variants use the same training and evaluation settings.

\noindent\textbf{Architectural components.}
We replace either the steady or transient operator with a convolutional operator and substitute independent self-attention for cross-branch attention. As shown in Table~\ref{tab:ablation_x2}, replacing global Fourier modeling causes a dramatic PSNR drop of \(5.13\) dB on Urban100 and \(6.18\) dB on Manga109, demonstrating that it is indispensable for preserving image-wide dependencies and long-range structural consistency. Similarly, ablating local complex-frequency modeling yields drops of \(0.11\) dB and \(0.08\) dB, supporting its complementary contribution to modeling localized, content-dependent aperiodic variations. Replacing cross-branch attention decreases PSNR by \(0.06\) dB and \(0.02\) dB, confirming the efficacy of collaborative aggregation.

\noindent\textbf{Key mechanisms.}
 The removal of the Fourier-magnitude loss function causes a PSNR decline of 0.12 dB across both datasets, affirming that spectral consistency complements pixel-domain supervision. Constraining the learnable transient poles to \(\operatorname{Re}(\mu)<0\) lowers Urban100 PSNR by \(0.10\) dB. Allowing unconstrained real parts permits both increasing and decaying spatial envelopes within finite windows, thereby providing greater flexibility for modeling local aperiodic variations.

\noindent\textbf{Feature responses.}
Fig.~\ref{fig:feature_vis_005} illustrates the complementary behavior of the two responses: the steady response preserves image-wide structural consistency, whereas the transient response emphasizes localized, content-dependent aperiodic variations. Their fusion through STCA further enhances local edge responses while retaining the overall structural organization.

\section{Conclusion}

We have presented LaST-SR, a Laplace-inspired steady--transient framework for jointly modeling image-wide structural information and localized aperiodic variations in SISR. Under certain assumptions, we derive an approximate steady--transient decomposition for static two-dimensional feature maps, extending the pole--residue interpretation from dynamical systems to spatial feature modeling. The proposed CFD module models the two responses through global full-spectrum Fourier and window-conditioned local complex-frequency branches, while STCA enables their collaborative aggregation. Extensive experiments and ablation studies demonstrate the effectiveness of LaST-SR. Future work will explore super-resolution under more complex degradations.

\appendix

\section{Steady--Transient Decomposition for Static Two-Dimensional Image Features}
\label{app:theory}

This section provides a detailed derivation of the approximate steady--transient decomposition introduced in the main paper. We first state the modeling conventions underlying the formulation and then derive the corresponding two-dimensional pole--residue representation.

\subsection{Continuous Feature Representation and Periodic-Extension Convention}
\label{app:problem}

Let $\mathbf{X}\in\mathbb{R}^{B\times H\times W\times C_{\mathrm{mid}}}$ denote an intermediate feature tensor. For the $b$-th sample, we choose a continuous interpolant $\mathbf{v}_b\in C(\overline{\Omega};\mathbb{R}^{C_{\mathrm{mid}}})\subset L^2(\Omega;\mathbb{R}^{C_{\mathrm{mid}}})$ satisfying
\begin{equation}
\mathbf{v}_b(x_w,y_h)=\mathbf{X}_{b,h,w,:},
\qquad
1\leq h\leq H,\quad 1\leq w\leq W,
\end{equation}
where $\Omega=[0,T_x)\times[0,T_y)$ and $(x_w,y_h)\in\Omega$ denotes a point on the regular spatial grid. This interpolant is introduced solely for continuous-domain analysis and is not part of the network implementation. Consistent with the scalar derivation in the main paper, we subsequently fix one sample and one feature channel, denote the corresponding scalar component by $v:\Omega\to\mathbb{R}$, and omit the batch and channel indices. The multi-channel formulation extends the scalar derivation by assigning complex-valued parameters to each input--output channel pair and aggregating the resulting responses over the input channels.

\begin{assumption}[Periodic-extension convention]
\label{ass:periodic_trunc}
We regard $\Omega$ as a fundamental domain of the two-dimensional torus
\begin{equation}
\mathbb{T}^2\cong\mathbb{R}^2/(T_x\mathbb{Z}\times T_y\mathbb{Z})
\end{equation}
and let $\widetilde v\in L^2(\mathbb{T}^2)$ denote the resulting periodic representative. This construction is introduced only to define a global Fourier representation. It neither assumes that the original image feature is physically periodic nor requires matching values on opposite boundaries of $\Omega$. Any boundary mismatch may introduce jump discontinuities at the periodic seams, which remain admissible in the $L^2$ setting.
\end{assumption}

For $p,q\in\mathbb{Z}$, define the Fourier coefficients of $\widetilde v$ as
\begin{equation}
\label{eq:app_alpha}
\alpha_{pq}
=
\frac{1}{T_xT_y}
\int_{0}^{T_x}
\int_{0}^{T_y}
\widetilde v(x,y)
\exp\!\left[-\mathrm{i}\left(\omega_p^x x+\omega_q^y y\right)\right]
\,dy\,dx,
\end{equation}
where
\begin{equation}
\omega_p^x=\frac{2\pi p}{T_x},
\qquad
\omega_q^y=\frac{2\pi q}{T_y}.
\end{equation}
The corresponding rectangular partial Fourier sum is
\begin{equation}
S_{P,Q}\widetilde v(x,y)
=
\sum_{p=-P}^{P}
\sum_{q=-Q}^{Q}
\alpha_{pq}
\exp\!\left(\mathrm{i}\omega_p^x x+\mathrm{i}\omega_q^y y\right).
\end{equation}
By completeness of the Fourier basis in $L^2(\mathbb{T}^2)$,
\begin{equation}
\lim_{P,Q\to\infty}
\left\|
\widetilde v-S_{P,Q}\widetilde v
\right\|_{L^2(\mathbb{T}^2)}
=0.
\end{equation}
Equivalently, for any $\varepsilon>0$, there exist $P_0,Q_0\in\mathbb{N}$ such that
\begin{equation}
\left\|
\widetilde v-S_{P,Q}\widetilde v
\right\|_{L^2(\mathbb{T}^2)}
<\varepsilon
\end{equation}
for all $P\geq P_0$ and $Q\geq Q_0$.

For the subsequent derivation, let $v_{P,Q}$ denote the restriction of $S_{P,Q}\widetilde v$ to $\Omega$. For notational simplicity, we write $v$ in place of $v_{P,Q}$. The resulting truncated representation is
\begin{equation}
\label{eq:app_fourier}
v(x,y)
=
\sum_{p=-P}^{P}
\sum_{q=-Q}^{Q}
\alpha_{pq}
\exp\!\left(\mathrm{i}\omega_p^x x+\mathrm{i}\omega_q^y y\right),
\qquad
(x,y)\in\Omega.
\end{equation}
Because the considered feature channel is real-valued, its Fourier coefficients satisfy
\begin{equation}
\alpha_{-p,-q}
=
\overline{\alpha_{pq}}.
\end{equation}

\subsection{Derivation of the Steady--Transient Decomposition}
\label{app:decomposition}

Under the truncated Fourier representation introduced above, we first consider a linear kernel integral mapping acting on the input feature $v$. For $(x,y)\in\Omega$, its general finite-domain form is
\begin{equation}
\label{eq:general_kernel_operator}
u_{\Omega}(x,y)
=
\iint_{\Omega}
\kappa\bigl((x,y),(\xi,\eta)\bigr)
v(\xi,\eta)
\,d\xi\,d\eta.
\end{equation}
Following the convolutional kernel mapping adopted in LNO~\cite{LNO}, we further restrict the kernel to depend only on the coordinate difference:
\begin{equation}
\label{eq:difference_kernel}
\kappa\bigl((x,y),(\xi,\eta)\bigr)
=
\kappa(x-\xi,y-\eta).
\end{equation}
Motivated by the complex-exponential representation of Prony-SS~\cite{prony_ss} and the Laplace-domain pole--residue parameterization of LNO~\cite{LNO}, we adopt the following positive-quadrant extension and one-sided-support convention to establish a two-dimensional one-sided Laplace representation.

\begin{assumption}[Positive-quadrant extension and one-sided support]
\label{ass:causal_extend}
Let $v$ be the truncated Fourier representation defined in the preceding subsection. Its positive-quadrant extension is
\begin{equation}
\label{eq:positive_extension}
v_{+}(x,y)
=
\mathbf{1}_{[0,\infty)^2}(x,y)
\sum_{p=-P}^{P}
\sum_{q=-Q}^{Q}
\alpha_{pq}
\exp\!\left(\mathrm{i}\omega_p^x x+\mathrm{i}\omega_q^y y\right).
\end{equation}
The translation-invariant kernel is represented by a first-quadrant-supported expansion of learnable complex-exponential modes:
\begin{equation}
\label{eq:app_kernel}
\kappa_{+}(a,b)
=
\mathbf{1}_{[0,\infty)^2}(a,b)
\sum_{m=1}^{K_x}
\sum_{n=1}^{K_y}
\mathcal{R}_{mn}
\exp\!\left(\mu_m^x a+\mu_n^y b\right),
\end{equation}
where $\mu_m^x,\mu_n^y\in\mathbb{C}$ are learnable complex poles along the two spatial directions and $\mathcal{R}_{mn}\in\mathbb{C}$ is the residue coefficient associated with each pole pair.
\end{assumption}

Under Assumption~\ref{ass:causal_extend}, the finite-domain response associated with the first-quadrant-supported kernel is
\begin{equation}
\label{eq:finite_domain_conv}
u_{\Omega}(x,y)
=
\iint_{\Omega}
\kappa_{+}(x-\xi,y-\eta)
v(\xi,\eta)
\,d\xi\,d\eta.
\end{equation}
The positive-quadrant extension and one-sided support are introduced only to establish a two-dimensional one-sided Laplace representation and do not assign temporal causality to the spatial image coordinates. For $(x,y)\in[0,\infty)^2$, define the one-sided convolution response on the positive quadrant as
\begin{equation}
\label{eq:positive_quadrant_conv}
u_{+}(x,y)
=
\int_{0}^{\infty}
\int_{0}^{\infty}
\kappa_{+}(x-\xi,y-\eta)
v_{+}(\xi,\eta)
\,d\eta\,d\xi.
\end{equation}
Because $\kappa_{+}(x-\xi,y-\eta)=0$ whenever $\xi>x$ or $\eta>y$, Eq.~\eqref{eq:positive_quadrant_conv} reduces to
\begin{equation}
\label{eq:causal_conv}
u_{+}(x,y)
=
\int_{0}^{x}
\int_{0}^{y}
\kappa_{+}(x-\xi,y-\eta)
v_{+}(\xi,\eta)
\,d\eta\,d\xi.
\end{equation}
For every $(x,y)\in\Omega$, the integration region $[0,x]\times[0,y]$ is contained in $\Omega$, and $v_{+}=v$ on this region. Consequently,
\begin{equation}
u_{\Omega}
=
u_{+}\big|_{\Omega}.
\end{equation}
We analyze $u_{+}$ on the positive quadrant below and omit the subscript ``$+$'' for notational simplicity.

For a function $f$ defined on $[0,\infty)^2$ and satisfying the corresponding exponential-order conditions, its two-dimensional one-sided Laplace transform is defined as
\begin{equation}
\mathcal{L}_2\{f\}(s_x,s_y)
=
\int_{0}^{\infty}
\int_{0}^{\infty}
f(x,y)
\exp\!\left[-(s_xx+s_yy)\right]
\,dy\,dx.
\end{equation}
The two-dimensional one-sided convolution theorem gives
\begin{equation}
\label{eq:app_laplace_product}
U(s_x,s_y)
=
K(s_x,s_y)V(s_x,s_y),
\end{equation}
where
\begin{equation}
V=\mathcal{L}_2\{v_{+}\},
\qquad
K=\mathcal{L}_2\{\kappa_{+}\},
\qquad
U=\mathcal{L}_2\{u_{+}\}.
\end{equation}
Because $v_{+}$ and $\kappa_{+}$ are finite linear combinations of complex exponentials on the positive quadrant, their Laplace transforms are
\begin{align}
V(s_x,s_y)
&=
\sum_{p=-P}^{P}
\sum_{q=-Q}^{Q}
\frac{\alpha_{pq}}
{(s_x-\mathrm{i}\omega_p^x)(s_y-\mathrm{i}\omega_q^y)},
\label{eq:app_V}\\
K(s_x,s_y)
&=
\sum_{m=1}^{K_x}
\sum_{n=1}^{K_y}
\frac{\mathcal{R}_{mn}}
{(s_x-\mu_m^x)(s_y-\mu_n^y)}.
\label{eq:app_K}
\end{align}
These expressions hold in the common region of convergence
\begin{align}
\operatorname{Re}(s_x)
&>
\max\!\left\{
0,\max_{1\leq m\leq K_x}
\operatorname{Re}(\mu_m^x)
\right\},
\label{eq:roc_x}\\
\operatorname{Re}(s_y)
&>
\max\!\left\{
0,\max_{1\leq n\leq K_y}
\operatorname{Re}(\mu_n^y)
\right\}.
\label{eq:roc_y}
\end{align}
Within this common region, the exponentially weighted absolute values of $v_{+}$ and $\kappa_{+}$ are integrable. Hence, the relevant multiple integrals are absolutely convergent, and the convolution theorem applies.

\begin{assumption}[Noncoincident Fourier and learnable poles]
\label{ass:distinct_poles}
The learnable poles do not coincide with the retained Fourier poles:
\begin{equation}
\mu_m^x\neq \mathrm{i}\omega_p^x,
\qquad
\mu_n^y\neq \mathrm{i}\omega_q^y
\end{equation}
for all valid indices. This condition ensures nonzero denominators in the subsequent partial-fraction expansion and excludes second-order poles caused by coincidences between the Fourier and learnable pole families. Pairwise distinctness among learnable poles is not required; repeated terms may be merged by summing their coefficients.
\end{assumption}

For compactness, define
\begin{equation}
d_{pm}^x
=
\mathrm{i}\omega_p^x-\mu_m^x,
\qquad
d_{qn}^y
=
\mathrm{i}\omega_q^y-\mu_n^y.
\end{equation}
\begin{proposition}[Exact four-family expansion of the truncated one-sided model]
\label{prop:exact_four_family}
Under Assumptions~\ref{ass:causal_extend} and~\ref{ass:distinct_poles}, the truncated one-sided convolution response admits, for $(x,y)\in[0,\infty)^2$, the exact expansion
\begin{align}
u(x,y)
&=
\sum_{p=-P}^{P}
\sum_{q=-Q}^{Q}
\lambda_{pq}
\exp\!\left(\mathrm{i}\omega_p^x x+\mathrm{i}\omega_q^y y\right)
\nonumber\\
&\quad+
\sum_{p=-P}^{P}
\sum_{n=1}^{K_y}
\rho_{pn}^{F\mu}
\exp\!\left(\mathrm{i}\omega_p^x x+\mu_n^y y\right)
\nonumber\\
&\quad+
\sum_{m=1}^{K_x}
\sum_{q=-Q}^{Q}
\rho_{mq}^{\mu F}
\exp\!\left(\mu_m^x x+\mathrm{i}\omega_q^y y\right)
\nonumber\\
&\quad+
\sum_{m=1}^{K_x}
\sum_{n=1}^{K_y}
\gamma_{mn}
\exp\!\left(\mu_m^x x+\mu_n^y y\right),
\label{eq:exact_spatial_response}
\end{align}
where
\begin{align}
\lambda_{pq}
&=
\alpha_{pq}
\sum_{m=1}^{K_x}
\sum_{n=1}^{K_y}
\frac{\mathcal{R}_{mn}}
{d_{pm}^x d_{qn}^y},
\label{eq:app_lambda}\\
\rho_{pn}^{F\mu}
&=
-
\sum_{q=-Q}^{Q}
\sum_{m=1}^{K_x}
\frac{\alpha_{pq}\mathcal{R}_{mn}}
{d_{pm}^x d_{qn}^y},
\label{eq:app_rho_Fmu}\\
\rho_{mq}^{\mu F}
&=
-
\sum_{p=-P}^{P}
\sum_{n=1}^{K_y}
\frac{\alpha_{pq}\mathcal{R}_{mn}}
{d_{pm}^x d_{qn}^y},
\label{eq:app_rho_muF}\\
\gamma_{mn}
&=
\mathcal{R}_{mn}
\sum_{p=-P}^{P}
\sum_{q=-Q}^{Q}
\frac{\alpha_{pq}}
{d_{pm}^x d_{qn}^y}.
\label{eq:app_gamma}
\end{align}
Moreover, the finite-domain response satisfies $u_{\Omega}=u|_{\Omega}$.
\end{proposition}

\begin{proof}
Substituting Eqs.~\eqref{eq:app_V} and~\eqref{eq:app_K} into Eq.~\eqref{eq:app_laplace_product} gives
\begin{equation}
\label{eq:app_product_expansion}
\begin{aligned}
U(s_x,s_y)
&=
\sum_{p=-P}^{P}
\sum_{q=-Q}^{Q}
\sum_{m=1}^{K_x}
\sum_{n=1}^{K_y}
\\[-2pt]
&\quad\times
\frac{\alpha_{pq}\mathcal{R}_{mn}}
{(s_x-\mathrm{i}\omega_p^x)(s_x-\mu_m^x)}
\frac{1}
{(s_y-\mathrm{i}\omega_q^y)(s_y-\mu_n^y)}.
\end{aligned}
\end{equation}
Assumption~\ref{ass:distinct_poles} ensures that $d_{pm}^x d_{qn}^y\neq0$. Applying one-dimensional partial fractions along both spatial directions yields
\begin{equation}
\label{eq:four_family_pf}
\begin{aligned}
&\frac{1}
{(s_x-\mathrm{i}\omega_p^x)(s_x-\mu_m^x)
 (s_y-\mathrm{i}\omega_q^y)(s_y-\mu_n^y)}=
\\
&
\frac{1}{d_{pm}^x d_{qn}^y}
\left(
\frac{1}{s_x-\mathrm{i}\omega_p^x}
-
\frac{1}{s_x-\mu_m^x}
\right)
\left(
\frac{1}{s_y-\mathrm{i}\omega_q^y}
-
\frac{1}{s_y-\mu_n^y}
\right).
\end{aligned}
\end{equation}
Expanding Eq.~\eqref{eq:four_family_pf} and regrouping terms with identical two-dimensional pole pairs gives
\begin{align}
U(s_x,s_y)
&=
\sum_{p=-P}^{P}
\sum_{q=-Q}^{Q}
\frac{\lambda_{pq}}
{(s_x-\mathrm{i}\omega_p^x)(s_y-\mathrm{i}\omega_q^y)}
\nonumber\\
&\quad+
\sum_{p=-P}^{P}
\sum_{n=1}^{K_y}
\frac{\rho_{pn}^{F\mu}}
{(s_x-\mathrm{i}\omega_p^x)(s_y-\mu_n^y)}
\nonumber\\
&\quad+
\sum_{m=1}^{K_x}
\sum_{q=-Q}^{Q}
\frac{\rho_{mq}^{\mu F}}
{(s_x-\mu_m^x)(s_y-\mathrm{i}\omega_q^y)}
\nonumber\\
&\quad+
\sum_{m=1}^{K_x}
\sum_{n=1}^{K_y}
\frac{\gamma_{mn}}
{(s_x-\mu_m^x)(s_y-\mu_n^y)},
\label{eq:exact_four_family}
\end{align}
where the coefficients are those in Eqs.~\eqref{eq:app_lambda}--\eqref{eq:app_gamma}. Since all sums are finite and the transforms share the common region of convergence in Eqs.~\eqref{eq:roc_x} and~\eqref{eq:roc_y}, the inverse Laplace transform can be applied term by term. Using
\begin{equation}
\mathcal{L}_2^{-1}
\left\{
\frac{1}{(s_x-a)(s_y-b)}
\right\}
=
\exp(ax+by)
\end{equation}
for each pole pair yields Eq.~\eqref{eq:exact_spatial_response}. The identity $u_{\Omega}=u|_{\Omega}$ follows from the preceding restriction argument, completing the proof.
\end{proof}

To obtain the dual-branch representation used in the main paper, we define the Fourier--Fourier pole family as the steady response, the complex-pole--complex-pole family as the transient response, and collect the two mixed pole families into an explicit residual:
\begin{equation}
\label{eq:exact_response_split}
u(x,y)
=
u_{\mathrm{ss}}(x,y)
+
u_{\mathrm{tr}}(x,y)
+
r_{\mathrm{mix}}(x,y),
\end{equation}
where
\begin{align}
u_{\mathrm{ss}}(x,y)
&=
\sum_{p=-P}^{P}
\sum_{q=-Q}^{Q}
\lambda_{pq}
\exp\!\left(\mathrm{i}\omega_p^x x+\mathrm{i}\omega_q^y y\right),
\label{eq:app_steady_response}\\
u_{\mathrm{tr}}(x,y)
&=
\sum_{m=1}^{K_x}
\sum_{n=1}^{K_y}
\gamma_{mn}
\exp\!\left(\mu_m^x x+\mu_n^y y\right),
\label{eq:app_transient_response}
\end{align}
and
\begin{align}
r_{\mathrm{mix}}(x,y)
&=
\sum_{p=-P}^{P}
\sum_{n=1}^{K_y}
\rho_{pn}^{F\mu}
\exp\!\left(\mathrm{i}\omega_p^x x+\mu_n^y y\right)
\nonumber\\
&\quad+
\sum_{m=1}^{K_x}
\sum_{q=-Q}^{Q}
\rho_{mq}^{\mu F}
\exp\!\left(\mu_m^x x+\mathrm{i}\omega_q^y y\right).
\label{eq:app_mixed_residual}
\end{align}
The two mixed families employ different pole types along the two spatial directions and do not generally vanish. For the SISR setting considered in this work, we focus on two complementary response families: globally supported oscillatory modes for representing repetitive structures, long-range dependencies, and overall structural consistency, and general complex-frequency modes that allow spatially varying envelopes for representing content-dependent aperiodic variations, with local adaptivity further enhanced by the subsequent windowed implementation. The Fourier--Fourier and complex-pole--complex-pole responses correspond directly to these two functional roles, whereas the mixed families describe directionally hybrid responses that use heterogeneous frequency mechanisms along the two axes. To obtain a structurally symmetric and parameter-efficient dual-branch representation aligned with these modeling objectives, we explicitly retain the two same-type pole families and collect the mixed families into $r_{\mathrm{mix}}$. Omitting this explicit residual yields the structured approximation adopted in the main paper:
\begin{equation}
\label{eq:app_approx_decomposition}
u(x,y)
\approx
u_{\mathrm{ss}}(x,y)
+
u_{\mathrm{tr}}(x,y).
\end{equation}
Here, ``steady'' and ``transient'' denote complementary spatial response families rather than temporal regimes. The derivation is carried out in the complexified feature space, and the real-valued network response is obtained by taking the real part of the resulting complex response.

\subsection{Fourier-Domain Form of the Convolution Response}
\label{app:freq_analysis}

The preceding subsection derives the pole--residue expansion of the convolution response through a two-dimensional one-sided Laplace transform. To explain why the retained steady response admits a global Fourier-domain representation, we now examine the corresponding periodic convolution operator under the periodic-extension convention. The purpose is not to rederive the steady--transient decomposition, but to show that convolution in the Fourier basis produces the same global mode family as the retained steady response.

Let $\widetilde v_{P,Q}:=S_{P,Q}\widetilde v\in L^2(\mathbb{T}^2)$ denote the truncated periodic input, and let $\widetilde\kappa\in L^1(\mathbb{T}^2)$ denote a periodic convolution kernel. Define the corresponding periodic convolution by
\begin{equation}
\label{eq:periodic_convolution}
u_{\mathrm{F}}(x,y)
=
\int_{0}^{T_x}
\int_{0}^{T_y}
\widetilde\kappa(x-\xi,y-\eta)
\widetilde v_{P,Q}(\xi,\eta)
\,d\eta\,d\xi,
\end{equation}
where the coordinate differences in $\widetilde\kappa$ are interpreted periodically. Because $\widetilde v_{P,Q}\in L^2(\mathbb{T}^2)$ and $\widetilde\kappa\in L^1(\mathbb{T}^2)$, the periodic convolution is well defined.

Let $\mathcal{F}_{\mathbb{T}^2}$ denote the Fourier-series transform on the two-dimensional torus, with the normalization
\begin{equation}
\begin{aligned}
\mathcal{F}_{\mathbb{T}^2}(f)[p,q]
&=
\frac{1}{T_xT_y}
\int_{0}^{T_x}
\int_{0}^{T_y}
f(x,y)
\\[-2pt]
&\qquad\times
\exp\!\left[-\mathrm{i}\left(\omega_p^x x+\omega_q^y y\right)\right]
\,dy\,dx.
\end{aligned}
\end{equation}
The two-dimensional periodic convolution theorem gives
\begin{equation}
\label{eq:fourier_convolution_theorem}
\mathcal{F}_{\mathbb{T}^2}(u_{\mathrm{F}})[p,q]
=
T_xT_y\,
\mathcal{F}_{\mathbb{T}^2}(\widetilde\kappa)[p,q]\,
\mathcal{F}_{\mathbb{T}^2}(\widetilde v_{P,Q})[p,q].
\end{equation}
By the orthogonality of the Fourier basis,
\begin{equation}
\mathcal{F}_{\mathbb{T}^2}(\widetilde v_{P,Q})[p,q]
=
\begin{cases}
\alpha_{pq},
&
|p|\leq P,\ |q|\leq Q,\\
0,
&
\text{otherwise}.
\end{cases}
\end{equation}
Define the effective kernel multiplier at mode $(p,q)$ as
\begin{equation}
\label{eq:fourier_kernel_multiplier}
H_{pq}
:=
T_xT_y\,
\mathcal{F}_{\mathbb{T}^2}(\widetilde\kappa)[p,q].
\end{equation}
For the retained modes, Eq.~\eqref{eq:fourier_convolution_theorem} becomes
\begin{equation}
\label{eq:fourier_coefficient_relation}
\mathcal{F}_{\mathbb{T}^2}(u_{\mathrm{F}})[p,q]
=
H_{pq}\alpha_{pq}.
\end{equation}
Applying the inverse Fourier-series transform gives
\begin{equation}
\label{eq:fourier_response_expansion}
u_{\mathrm{F}}(x,y)
=
\sum_{p=-P}^{P}
\sum_{q=-Q}^{Q}
H_{pq}\alpha_{pq}
\exp\!\left(\mathrm{i}\omega_p^x x+\mathrm{i}\omega_q^y y\right).
\end{equation}

Equation~\eqref{eq:fourier_response_expansion} shows that periodic convolution preserves the spatial Fourier basis functions and modulates only their coefficients through complex-valued mode multipliers. The retained steady response in Eq.~\eqref{eq:app_steady_response} is spanned by exactly the same global Fourier-mode family. More specifically, Eq.~\eqref{eq:app_lambda} gives
\begin{equation}
\lambda_{pq}
=
M_{pq}^{\mathrm{ss}}\alpha_{pq},
\end{equation}
where
\begin{equation}
\label{eq:steady_effective_multiplier}
M_{pq}^{\mathrm{ss}}
:=
\sum_{m=1}^{K_x}
\sum_{n=1}^{K_y}
\frac{\mathcal{R}_{mn}}
{d_{pm}^x d_{qn}^y}.
\end{equation}
Hence, the steady response itself admits the Fourier spectral-multiplier form
\begin{equation}
\label{eq:fno_general_form}
u_{\mathrm{ss}}
=
\mathcal{F}_{\mathbb{T}^2}^{-1}
\left(
\mathcal{M}_{P,Q}^{\mathrm{ss}}
\mathcal{F}_{\mathbb{T}^2}(\widetilde v_{P,Q})
\right),
\end{equation}
where $\mathcal{M}_{P,Q}^{\mathrm{ss}}$ acts mode-wise as
\begin{equation}
\left(
\mathcal{M}_{P,Q}^{\mathrm{ss}}\widehat v
\right)[p,q]
=
M_{pq}^{\mathrm{ss}}\widehat v[p,q],
\qquad
|p|\leq P,\quad |q|\leq Q.
\end{equation}
Equation~\eqref{eq:fno_general_form} has the same general form as the spectral convolution used in the Fourier neural operator~\cite{FNO}: the input is transformed to the Fourier domain, each frequency mode is modulated by a complex-valued spectral multiplier, and the result is mapped back to the spatial domain through the inverse Fourier transform. Because Fourier modes have global spatial support, this form provides a direct analytical basis for implementing the steady response with a global Fourier-domain operator.

\bibliographystyle{unsrtnat}
\bibliography{references}
\end{document}